\documentclass[sigconf, nonacm]{acmart}

\newcommand\vldbavailabilityurl{URL_TO_YOUR_ARTIFACTS}
\newcommand\vldbpagestyle{plain} 

\newtheorem{thm}{Theorem}[section]

\newtheorem{rem}[thm]{Remark}

\numberwithin{equation}{section}
\numberwithin{figure}{section}
\usepackage{algorithm, algorithmicx,algpseudocode}
\usepackage{appendix}
\usepackage{balance}
\usepackage{tikz}
\usepackage{subfigure}
\begin{document}
\title{Data valuation: The partial ordinal Shapley value for machine learning}

\renewcommand{\thefootnote}{\fnsymbol{footnote}}


\author{Jie Liu}
\authornotemark[1]
\affiliation{%
  \institution{Zhejiang University}
  \streetaddress{P.O. Box 1212}
  \city{Hangzhou}
  \state{China}
  \postcode{43017-6221}
}
\email{liujiemath@hotmail.com}

\author{Peizheng Wang}
\authornotemark[1]

\orcid{0000-0002-1825-0097}
\affiliation{%
  \institution{Zhejiang University}
  \streetaddress{1 Th{\o}rv{\"a}ld Circle}
  \city{Hangzhou}
  \country{China}
}
\email{wangpeizheng@zju.edu.cn}

\author{Chao Wu}

\authornotemark[2]
\orcid{0000-0001-5109-3700}
\affiliation{%
  \institution{Zhejiang University}
  \city{Hangzhou}
  \country{China}
}
\email{chao.wu@zju.edu.cn}

\begin{abstract}
Data valuation using Shapley value has emerged as a prevalent research domain in machine learning applications. However, it is a challenge to address the role of order in data cooperation as most research lacks such discussion. To tackle this problem, this paper studies the definition of the partial ordinal Shapley value by group theory in abstract algebra. Besides, since the calculation of the partial ordinal Shapley value requires exponential time, this paper also gives three algorithms for approximating the results. The Truncated Monte Carlo algorithm is derived from the classic Shapley value approximation algorithm. The Classification Monte Carlo algorithm and the Classification Truncated Monte Carlo algorithm are based on the fact that the data points in the same class provide similar information, then we can accelerate the calculation by leaving out some data points in each class.
\end{abstract}

\maketitle

\footnotetext{*Both authors contributed equally to the paper}
\footnotetext{\dag Corresponding author}

\pagestyle{\vldbpagestyle}

\ifdefempty{\vldbavailabilityurl}{}{
\vspace{.3cm}
\begingroup\small\noindent\raggedright\textbf{Source code:}\\
The source code has been made available at \url{https://github.com/PeizhengWang/PartialOrdinalShapley}.
\endgroup
}

\section{Introduction}

Data valuation using machine learning has become prevalent in numerous research domains. A variety of applications are enhanced by data pricing\cite{JiachenTWang,yoon2020data,chen2019towards,agarwal2019marketplace,cong2022data,yan2021if,wu2022davinz}. In economics and game theory, the Shapley value\cite{Owen,castro2017improving,castro2009polynomial} stands as a conventional pricing approach, which can be effectively applied in the realm of data valuation. It consists of fruitful theoretical and experimental results which have been exploited in several fields\cite{Corder,kwon2021beta,huang2021shapley,wang2020shapley,jia2019efficient,tang2021data,covert2021improving,ghorbani2020distributional}. However, most of these results are based on the assumption that the order of data cooperation does not affect the whole coalition value, which is not always met in practical applications. For instance, in scenarios where there exists an ordinal cooperation between two entities, $A$ and $B$, the order in which each entity appears may carry additional influence besides their individual marginal contributions, i.e. 
\begin{equation}
    U(A,B)\neq U(B,A)
\end{equation}
where $U$ is the utility function and $A$ and $B$ are not same. This situation is very common in the recommendation system since the first ad is naturally more effective than the ad given later.

In \cite{Pérez-Castrillo}, the authors propose an approach to compute ordinal Shapley values, which entails iterating over the number of cooperative members. This procedure is commonly recognized as a challenging task in machine learning. To address this issue, an alternative definition of the ordinal Shapley value based on group theory is given in this paper. The permutation group is also mentioned in the classic Shapley values, but most research lacks discussions of the role of group theory in the classic Shapley values. The definition of the ordinal Shapley value given in this paper emphasizes the relation between symmetry in the permutation group and symmetry in the Shapley values. In addition, this paper also gives a simplified definition of the ordinal Shapley value, i.e. the partial ordinal Shapley value, which makes its application in machine learning possible.

As with the classic Shapley value, the partial ordinal Shapley value is also computationally expensive. The approximation of the classic Shapley value in data pricing has been exploited in lots of papers. Also, some algorithms can be applied to the partial ordinal Shapley value, such as Monte Carlo sampling\cite{castro2009polynomial,jia2019towards}, Truncated Monte Carlo(TMC) sampling\cite{Zou}, etc. These methods operate by iteratively adding data points and computing marginal contributions. In addition to the algorithms derived from the classic Shapley value, we also give two algorithms based on classification, namely Classification Monte Carlo(CMC) and Classification Truncated Monte Carlo(CTMC). The motivation of these algorithms is that the data points of the same class are similar, which leads to less marginal contribution. By selecting only a subset of representative data points from each class, these algorithms can reduce computation costs without the performance drop.

{\bf Our contributions} We give the definition of the ordinal Shapley value for quantifying the data value in the ordinal coalition. We redefine the four axioms in the classic Shapley value and the allocation function is given. Besides, we study the definition of the partial ordinal Shapley value, which makes data valuation in machine learning possible. To the best of our knowledge, we are the first to study data valuation in the ordinal situation. We also study two special cases and give the corresponding allocation functions based on the group theory.

In this paper, three algorithms for approximating the partial ordinal Shapley value are given and we compare these algorithms in Wine, Cancer, and Adult datasets. We also calculate the error analysis of TMC and CMC in Appendix.

{\bf Outlook} There are many future research directions in this area. For instance, the approximation for the ordinal Shapley value is unknown, which is a potential research project for optimizing the data valuation. 

Also, the special cases of the partial ordinal Shapley value might play an important role in data valuation, especially in distributed machine learning areas such as federated learning\cite{fan2022fair,fan2022improving,wang2020principled} and blockchain\cite{shen2020blockchain,zhu2019incentive}. The classic Shapley value of the partially defined cooperative game is given in \cite{masuya2021approximated}, where the author only considers some specific coalitions. So there might be some relations between the corresponding ordinal or partial ordinal situation and the special cases we have studied in Subsection \ref{subsec:specialcases}.
\section{Preliminaries}
\subsection{Utility function and marginal contribution}

The classic data valuation where the order of the data points is not taken into consideration for machine learning aims to train the set of data points and calculate their contributions. The contribution can be reflected as a function, i.e. the utility function. Consider a dataset $N=\{1,2,\ldots,n\}$ where $i$($i=1,2,\ldots,n$) is the client in $N$, we define $U: 2^N\mapsto \mathbb{R}$ as the utility function. In this function, the independent variable is the subset $S\subset N$ and the dependent variable is a real number, which is the data value of $S$.

For ordinal data valuation, the utility function can be defined analogously. Besides the attendance of each client, the position also plays a role, i.e. the same cooperative clients with different permutations give different utility function results. In this case, the independent variable in utility function $U$ is a sequence $\mathfrak{S}$ for which the elements are the clients in $N$. 

Most frameworks for data valuation are based on marginal contribution. In the classic data valuation, we calculate the marginal contribution of the client $i$ under the subset $S\subset N\backslash \{i\}$ by $U(S\bigcup\{i\})-U(S)$. However, in the ordinal data valuation, the ordinal marginal contribution of the client $i$ is $U(\mathfrak{S}_{i,k})-U(\mathfrak{S})$. In this expression, $\mathfrak{S}_{i,k}$ is the sequence where we insert $i$ into $k$-th place.

\subsection{Group Theory}
We recall some knowledge from \cite{Rotman}. Let $\Pi$ be the permutation group $S_n$. For $S\subset N$, we define $\Pi|_S$ as the group with $\Pi$ being restricted to $S$. Also, we define $\bar{S}=N\backslash S$ and $\Pi_S=\Pi|_S\times \Pi|_{\bar{S}}$ as a subgroup of $\Pi$.

For any $S\subset N$ and the permutation $\pi_S\in\Pi_S$, we define
\begin{equation}
\alpha_{ij}\pi_S=\begin{cases}\pi_S,\quad i,j\in S\text{ or }i,j\notin S\\
(i,j)\pi_S(i,j),\quad\text{other}
\end{cases}
\end{equation}
where $(i,j)$ is a transposition in $\Pi$. Also, we define $S^{(i,j)}$ as a subset of $N$ where if the client $i\in S$ and $j\notin S$, then $i$ is replaced by the client $j$, and if $j\in S$ and $i\notin S$, then $j$ is replaced by $i$ respectively.

We define the sequence $\mathfrak{S}$ and $\mathfrak{S}^{(i,j)}$ as
\begin{equation}
    \mathfrak{S}=(s_1,s_2,\cdots,s_{|S|}),\quad s_{i}\in S,s_{i}<s_{i+1}
\end{equation}
and $\mathfrak{S}^{(i,j)}=(i,j)(\mathfrak{S})$, then the utility function $U(\mathfrak{S})$ can be defined analogously.

Fixed $\pi|_S(\mathfrak{S})=(s_1^\pi,s_2^\pi,\cdots,s_{|S|}^\pi)$ where $S\subset N\backslash\{i\}$ and $\pi|_S\in\Pi|_S$, we define 
\begin{equation}
    \pi|_{i,k}(\mathfrak{S})=(s_1^\pi,s_2^\pi,\cdots,s_k^\pi,i,s_{k+1}^\pi,\cdots,s_{|S|}^\pi),\quad k=0,1,\ldots,|S|,
\end{equation}
i.e. $k$ is the number of the elements precede $i$ in $\pi|_{i,k}(\mathfrak{S})$. 

\section{Shapley Value}
\subsection{From the classic Shapley value to the partial ordinal Shapley value}
This section aims to give a generalization from the classic Shapley value to the partial ordinal Shapley value. We recall the definition of the classic Shapley value in \cite{Owen}.

{\bf Classic Shapley value}
Given a dataset $N=\{1,2,\ldots,n\}$ and a utility function $U(S)$, a unique allocation $\phi_i(U)$($i\in N$) is obtained with the following four axioms
\begin{itemize}
    \item {\bf Null Player} If $U(S)=U(S\bigcup\{i\})$ for all $S\subset N\backslash\{i\}$, then $\phi_i(U)=0$.
    \item {\bf Symmetry} If $U(S\bigcup\{i\})=U(S\bigcup\{j\})$ for all $S\subset N\backslash\{i,j\}$, then $\phi_i(U)=\phi_j(U)$.
    \item {\bf Efficiency} $\underset{i=1}{\stackrel{n}{\sum}}\phi_i(U)=U(N)$.
    \item {\bf Additivity} For two utility functions $U_1$ and $U_2$, we have $\phi_i(U_1+U_2)=\phi_i(U_1)+\phi_i(U_2)$.
\end{itemize}
By abuse of notation, we use $\phi_i=\phi_i(U)$. The unique allocation function is
\begin{equation}
\phi_i=\frac{1}{n}\underset{S\subset N\backslash\{i\}}{\sum}\frac{1}{\left(
    \begin{array}{l}
    n-1\\|S|
    \end{array}\right)}(U(S\bigcup\{i\})-U(S))
\end{equation}
and the alternative expression is
\begin{equation}\label{classicshapley}
    \phi_i=\frac{1}{n!}\underset{\pi\in\Pi}{\sum}[U(P_i^\pi\bigcup\{i\})-U(P_i^\pi)]
\end{equation}
where $P_i^\pi$ is the set of clients which precede client $i$ in $\pi(1,2,\ldots,n)$.

The classic Shapley value is independent of the order of cooperation. However, this assumption is not always met in data valuation applications. In this case, we generalize the classic Shapley value to the ordinal Shapley value in a group theory setting. Note that the definition of the ordinal Shapley value is based on the ordinal marginal contribution.

{\bf Ordinal Shapley value} In terms of the generalization from the classic Shapley value to the ordinal Shapley value, the corresponding four axioms can be replaced by
\begin{itemize}
    \item {\bf Ordinal Null Player} For any subset $S\subset N\backslash\{i\}$ and the corresponding sequence $\mathfrak{S}$ and permutation $\pi|_S\in\Pi|_S$, if
    \begin{equation}
        U(\pi|_S(\mathfrak{S}))=U(\pi|_{i,k}(\mathfrak{S})),\quad k=0,1,2,\ldots,|S|,
    \end{equation}
    we have $\phi_i=0$.
    \item {\bf Ordinal Symmetry} For $i,j\in N$, if
    \begin{equation}
    U(\pi|_S(\mathfrak{S}))=U((\alpha_{i,j}\pi_S)|_{S^{(i,j)}}(\mathfrak{S}^{(i,j)}))
    \end{equation}
    for any $S\subset N$ and the corresponding sequence $\mathfrak{S}$ and $\pi|_S\in\Pi|_S$, we have $\phi_i=\phi_j$.
    \item {\bf Ordinal Efficiency} We have
    \begin{equation}
    \underset{i=1}{\stackrel{n}{\sum}} \phi_i=\frac{1}{n!}\underset{\pi\in\Pi}{\sum} U(\pi(1,2,3,\ldots,n)).
    \end{equation}
    \item {\bf Ordinal Additivity} For two ordinal utility functions $U_1$ and $U_2$, we have $\phi_i(U_1+U_2)=\phi_i(U_1)+\phi_i(U_2)$.
\end{itemize}
\begin{thm}
There exists an allocation function that satisfies these four axioms, i.e.
\begin{multline}\label{ordinalshapley}
\phi_i=\frac{1}{n}\underset{S\subset N\backslash\{i\}}{\sum}\frac{1}{(|S|+1)!\left(
    \begin{array}{l}
    n-1\\|S|\end{array}\right)}\\
    \underset{\pi\in\Pi|_S}{\sum}\underset{k=0}{\stackrel{|S|}{\sum}}(U(\pi|_{i,k}(\mathfrak{S})-U(\pi|_S(\mathfrak{S}))).
\end{multline}
\end{thm}
\begin{proof}
The ordinal null player, ordinal symmetry, and ordinal additivity properties of the ordinal Shapley value are obvious. We only prove the ordinal efficiency property. Given a sequence $\mathfrak{U}$ where $|\mathfrak{U}|=k\leq n-1$ and for any client $i$ included in $\mathfrak{U}$, the coefficient of $U(\mathfrak{U})$ in $\phi_i$ is $\frac{1}{n}\frac{1}{k!\left(
    \begin{array}{l}
    n-1\\k-1\end{array}\right)}$.
    For the client $i$ not included in $\mathfrak{U}$, the corresponding coefficient of $U(\mathfrak{U})$ in $\phi_i$ is $-\frac{1}{n}\frac{k+1}{(k+1)!\left(
    \begin{array}{l}
    n-1\\k\end{array}\right)}$.
So the coefficient of $U(\mathfrak{U})$ in $\underset{i=1}{\stackrel{n}{\sum}} \phi_i$ is
\begin{equation}
    \frac{1}{n}\frac{k}{k!\left(
    \begin{array}{l}
    n-1\\k-1\end{array}\right)}-\frac{1}{n}\frac{(n-k)(k+1)}{(k+1)!\left(
    \begin{array}{l}
    n-1\\k\end{array}\right)}=0.
\end{equation}
For the sequence $\mathfrak{U}$ where $|\mathfrak{U}|=n$, the coefficient of $U(\mathfrak{U})$ in $\phi_i$ is
\begin{equation}
    \frac{1}{n}\frac{1}{n!\left(
    \begin{array}{l}
    n-1\\n-1\end{array}\right)}=\frac{1}{n(n!)}
\end{equation}
and the ordinal efficiency is proved.
\end{proof}

Since this expression is difficult to approximate in machine learning, we consider the partial ordinal Shapley value, where the first axiom is replaced by
\begin{itemize}
    \item {\bf Partial Ordinal Null Player} For any subset $S\subset N\backslash\{i\}$ and the corresponding permutation $\pi|_S\in\Pi|_S$, if
    \begin{equation}
        U(\pi|_S(\mathfrak{S}))=U(\pi|_{i,|S|}(\mathfrak{S})),
    \end{equation}
    we have $\phi_i=0$.
\end{itemize}
\begin{thm}
The allocation function satisfying the partial ordinal null player, the ordinal symmetry, the ordinal efficiency, and the ordinal additivity is
\begin{equation}\label{partialordinalshapley}
\phi_i=\frac{1}{n}\underset{S\subset N\backslash\{i\}}{\sum}\frac{1}{|S|!\left(
    \begin{array}{l}
    n-1\\|S|\end{array}\right)}\underset{\pi\in\Pi|_S}{\sum}(U(\pi_{i,|S|}(\mathfrak{S}))-U(\pi|_S(\mathfrak{S})))
\end{equation}
i.e.
\begin{equation}\label{alterpartialordinaryshapley}
 \phi_i=\frac{1}{n!}\underset{\pi\in\Pi}{\sum}[U(\mathfrak{P}_i^\pi,i)-U(\mathfrak{P}_i^\pi)].
\end{equation}
In this equation, $\mathfrak{P}_i^\pi$ is the sequence of clients that precede $i$ in $\pi(1,2,\ldots,n)$ and $(\mathfrak{P}_i^\pi,i)$ is the sequence in which client $i$ included.
\end{thm}
\begin{proof}
    The proof of the ordinal efficiency in \eqref{partialordinalshapley} is similar to \eqref{ordinalshapley}. We only prove the equivalence between \eqref{partialordinalshapley} and \eqref{alterpartialordinaryshapley}. Note that the number of sequences $\pi(1,2,\ldots,n)$($\pi\in\Pi$) with $(\mathfrak{P}_i^\pi,i)$ preceded and $|(\mathfrak{P}_i^\pi,i)|=k$ is $(n-k)!$
and the coefficient of the corresponding ordinal marginal contribution in \eqref{partialordinalshapley} is
\begin{equation}
  \frac{1}{n}\frac{1}{(k-1)!\left(
    \begin{array}{l}
    n-1\\k-1\end{array}\right)}.  
\end{equation}
We have
\begin{equation}
    \frac{1}{n!}(n-k)!=\frac{1}{n}\frac{1}{(k-1)!\left(
    \begin{array}{l}
    n-1\\k-1\end{array}\right)}
\end{equation}
and this equivalence is proved.
\end{proof}
Comparing \eqref{classicshapley} with \eqref{alterpartialordinaryshapley}, it is obvious that the classic Shapley value is a special case of the partial ordinal Shapley value.
\begin{rem}
Note that the sufficiency of the ordinal null player and the partial ordinal null player are not satisfied, since in data valuation applications, the ordinal marginal contributions might be negative. However, if we add the assumption that the utility function is monotonic, i.e. all ordinal marginal contributions are non-negative, the sufficiency is met. In this case, the Shapley value of client $i$ is positive if and only if there exists at least one sequence $\pi|_S(\mathfrak{S})$ where $S\subset N\backslash\{i\}$ such that the corresponding ordinal marginal contribution is positive.

The sufficiency of the ordinal symmetry is not satisfied since two clients with the same Shapley value do not need to be symmetry for all sequences.
\end{rem}
\begin{rem}
The uniqueness of \eqref{ordinalshapley} and \eqref{partialordinalshapley} is not satisfied. However, since the classic Shapley value is a special case of the partial ordinal Shapley value, the general expression of the partial ordinal Shapley value can be written as
\begin{multline}
\phi_i=\frac{1}{n}\underset{S\subset N\backslash\{i\}}{\sum}\frac{1}{\left(
    \begin{array}{l}
    n-1\\|S|\end{array}\right)}\\
    \underset{\pi|_S\in\Pi(S)}{\sum}q(\pi|_S(\mathfrak{S}))(U(\pi|_{i,|S|}(\mathfrak{S})-U(\pi|_S(\mathfrak{S}))),\\ \quad \underset{\pi|_S\in\Pi|_S}{\sum}q(\pi|_S(\mathfrak{S}))=1
\end{multline}
with the condition
\begin{equation}
    \underset{i=1}{\stackrel{n}{\sum}} \phi_i=\frac{1}{n!}\underset{\pi\in\Pi}{\sum} U(\pi(1,2,3,\ldots,n))
    \end{equation}
i.e. the ordinal efficiency being met.
\end{rem}

\subsection{Special cases}\label{subsec:specialcases}
We assume the set of clients $N$ can be separated into a couple of unions, i.e. $N=U_1\bigcup U_2\bigcup\cdots\bigcup U_s$ with $U_i\bigcap U_j=\emptyset$, and the corresponding partition is
\begin{equation}
    \lambda=(t_1,t_2,\cdots,t_s), \quad t_j=|U_j|,\quad \underset{j=1}{\stackrel{s}{\sum}}t_j=n.
\end{equation}
We consider the following two special cases.
\begin{itemize}
    \item The transpositions in the union influence the final utility function
values, and the transpositions between the clients in the different unions do not influence the utility function values.
    \item The transpositions in the union do not influence the final utility function
values, and the transpositions between the clients in the different unions influence the utility function values.
\end{itemize}
We define $\Pi^{\lambda}\cong\Pi|_{U_1}\times\Pi|_{U_2}\cdots\Pi|_{U_s}$ as the subgroup of $\Pi$, and we define $\Pi^{\bar\lambda}=\Pi/\Pi^\lambda$ as the equivalent class. We define $\Pi^\lambda|_S$ as the group $\Pi^\lambda$ restricted to the subset $S\subset N$, and $\Pi^{\bar\lambda}|_S$ can be defined analogously.
\begin{thm}
The partial ordinal Shapley value in the first case can be rewritten as
\begin{multline}\label{specialshapley:case1}
    \phi_i=\frac{1}{n}\underset{S\subset N\backslash\{i\}}{\sum}\frac{1}{\left(
    \begin{array}{l}
    n-1\\|S|
    \end{array}\right)}\\
    \underset{\sigma\in\Pi^\lambda|_S}{\sum}\frac{1}{s_1!s_2!\cdots s_t!}(U(\sigma(\mathfrak{S}),i)-U(\sigma(\mathfrak{S}))),\\
    \quad s_i=|U_i\bigcap S|.
\end{multline}
The partial ordinal Shapley value in the second case can be rewritten as
\begin{multline}\label{specialshapley:case2}
    \phi_i=\frac{1}{n}\underset{S\subset N\backslash\{i\}}{\sum}\frac{1}{\left(
    \begin{array}{l}
    n-1\\|S|
    \end{array}\right)}\\
    \underset{\sigma\in\Pi^{\bar{\lambda}}|_S}{\sum}\frac{s_1!s_2!\cdots s_t!}{|S|!}(U(\sigma(\mathfrak{S}),i)-U(\sigma(\mathfrak{S}))),\\
    \quad s_i=|U_i\bigcap S|.
\end{multline}    
\end{thm}
\begin{proof}
  We only prove \eqref{specialshapley:case1} and the proof of \eqref{specialshapley:case2} is similar. Consider a subset $S\subset N$ and $s_i=|U_i\bigcap S|$, then all the permutations which influence the utility function values generate the group $\Pi^\lambda|_S\cong \Pi|_{U_1\bigcap S}\times\Pi|_{U_2\bigcap S}\times\cdots\times\Pi|_{U_s\bigcap S}$ and $|\Pi^\lambda|_S|=s_1!s_2!\cdots s_t!$. Besides, we have $|\Pi|_S/\Pi^\lambda|_S|=\frac{|S|!}{s_1!s_2!\cdots s_t!}$, i.e. the number of the permutations in $\Pi|_S$ sharing equal utility function values. So \eqref{specialshapley:case1} is proved.
\end{proof}
Note that if $\lambda=(1,1,\cdots,1)$, \eqref{specialshapley:case1} is equivalent to the classic Shapley value in \eqref{classicshapley}, and \eqref{specialshapley:case2} is equivalent to the partial ordinal Shapley value in \eqref{partialordinalshapley}. If $\lambda=(n)$, \eqref{specialshapley:case1} is equivalent to the partial ordinal Shapley value in \eqref{partialordinalshapley}, and \eqref{specialshapley:case2} is equivalent to the classic Shapley value in \eqref{classicshapley}.
\section{Methods}
The goal of this section is to give three algorithms for approximating the partial ordinal Shapley value and compare the efficiency.
\subsection{Truncated Monte Carlo}
Truncated Monte Carlo(TMC) is a canonical algorithm for approximating the classic Shapley value, which can be applied to the partial ordinal Shapley value. The motivation of the TMC algorithm is that the range of marginal contributions will be decreased with the increase of cooperative clients, and then the clients will be truncated if they bring less marginal contributions than the truncated factor. We use the algorithm mentioned in \cite{Zou} and the pseudocode is shown in Algorithm \ref{tmcalgo}.
\subsection{Classification Monte Carlo}
Besides the TMC algorithm, Classification Monte Carlo sampling (CMC) gives another way to accelerate the approximation of the partial ordinal Shapley value. We assume the dataset $N$ consists of a couple of classes, and the clients of the same class are similar, i.e. expanding the scale of cooperation by adding clients from the same class provides less marginal contribution than by adding clients from different classes. We select clients from each class with the same proportion in each round, then we calculate the partial ordinal Shapley value via Monte Carlo sampling. The pseudocode is shown in Algorithm \ref{cmcalgo}.

\begin{algorithm}[htb]
  \caption{ The TMC algorithm.}\label{tmcalgo}
  \begin{algorithmic}[1]
    \Require
      Train data $N=\{1,2,\ldots,n\}$, utility function $U$
    \Ensure
      Shapley value $\phi_i$($i=1,2,\ldots,n$)
    \State Initialize $\phi_i=0$($i=1,2,\cdots,n$), $t=0$;
    \While {Convergence criteria not met} \do;
    \State $t:=t+1$;
    \State $\pi^t$: Random permutation in $\Pi$;
    \State $v_0^t=0$;
    \For { $j\in\{1,2,\ldots,n\}$ } \do;
    \If {$|U(\pi^t(N))-v_{j-1}^t|<\mbox{Truncated factor}$}
    \State{$v_j^t=v_{j-1}^t$}
    \Else \State $v_j^t=U(\pi^t[1],\pi^t[2],\cdots,\pi^t[j])$;
    \EndIf
    \EndFor
    \State $\phi_{\pi^t[j]}=\frac{t-1}{t}\phi_{\pi^{t-1}[j]}+\frac{1}{t}(v_j^t-v_{j-1}^t)$;
    \EndWhile
    \State\Return $\phi_1,\phi_2,\cdots,\phi_n$;
  \end{algorithmic}
\end{algorithm}

\begin{algorithm}[htb]
  \caption{ The CMC algorithm.}\label{cmcalgo}
  \begin{algorithmic}[1]
    \Require
      Train data $N=\{1,2,\ldots,n\}$, utility function $U$
    \Ensure
      Shapley value $\phi_i$($i=1,2,\ldots,n$)
    \While {Convergence criteria not met} \do;
    \State Select clients from each class with the same proportion $q$, and the permutation group is $\Gamma$;
    \State $t:=t+1$;
    \State Random permutation $\gamma^t$ in $\Gamma$;
    \State $v_0^t=0$;
    \For {$j\in\{1,2,\ldots,\lfloor qn\rfloor\}$}
    \State $v_j^t=U(\gamma^t[1],\gamma^t[2],\cdots,\gamma^t[j])$;
    \EndFor
        \State $\phi_{\gamma^t[j]}=\frac{t-1}{t}\phi_{\gamma^{t-1}[j]}+\frac{1}{t}(v_j^t-v_{j-1}^t)$;
    \EndWhile
    \State \Return $\phi_1,\phi_2,\cdots,\phi_n$;
  \end{algorithmic}
\end{algorithm}

\subsection{Error analysis of TMC and CMC}
Now we give the error analysis of the TMC and the CMC algorithms. For $1\leq k\leq n$, we define
\begin{equation}
    \phi^{k}_i=\frac{1}{n}\underset{\substack{S\subset N\backslash\{i\}\\|S|=k}}{\sum}\frac{1}{|S|!\left(
    \begin{array}{l}
    n-1\\|S|\end{array}\right)}\underset{\pi\in\Pi|_S}{\sum}(U(\pi|_{i,|S|}(\mathfrak{S}))-U(\pi|_S(\mathfrak{S}))),
\end{equation}
\begin{equation}
    \phi^{(k)}_i=\underset{j=1}{\stackrel{k}{\sum}}\phi_i^k,
\end{equation}
and $\epsilon_k=|\phi_i^{(k)}-\phi_i|$.

    We denote $T$ as the sample size and $\hat{\phi}_i$, $\hat{\phi}_i^{k}$ and $\hat{\phi}_i^{(k)}$ as the corresponding approximation values. 
    
    In the ordinal cooperative game, the marginal contribution range depends both on the client and the cooperative position. In the TMC algorithm, we estimate the error with the condition that each position's bounding of marginal contribution is known. In the CMC algorithm, we estimate the error with the condition that the range of the marginal contribution provided by the fixed client is known. The calculation details are shown in Appendix \ref{appendix}.

{\bf TMC} Let 
\begin{equation}
 r_{max,k}=\max\{r_1,r_2,\cdots,r_k\}
\end{equation}
where $r_j\in\mathbb{R}^+$($j=1,2,\ldots,n$) is the range of the marginal contribution in $j$-th place. In the partial ordinal situation, we assume there exists $c\in\{1,2,\ldots,n\}$ such that $r_c\geq r_{c+1}\geq\cdots\geq r_n$. We denote the truncated factor as $\underset{j=k+1}{\stackrel{n}{\sum}}r_j<\sigma<\underset{j=k}{\stackrel{n}{\sum}}r_j$ with $k\geq c$, i.e. we only consider the permutations for which the client $i$ is in $1,2,\ldots,k$-th place. In this case, we denote the approximation value $\hat{\phi}_i$ as $\hat{\phi}^{(k)}_i$. Without loss of generality, we let $\epsilon\geq\epsilon_k$, and we have
\begin{equation}
    Pr(|\phi_i-\hat{\phi}_i^{(k)}|\geq\epsilon) \leq\underset{j=1}{\stackrel{k}{\sum}}2\exp(-\frac{2m_j(\epsilon-\epsilon_k)^2}{k^2 r_{max,k}^2}).
    \end{equation}
In this equation, $m_j$ is the number of samples where client $i$ is in $j$-th place.

{\bf CMC} Let $N=N_1\bigcup N_2\bigcup\cdots\bigcup N_s$ where $N_i\bigcap N_j=\emptyset$ for $i\neq j$ be the classification of $N$. The range of the marginal contribution of client $i$ is $\delta_i$. The probability of client $i$ in the class $N_j$ being chosen is $q_j$($j=1,2,\ldots,s$). We have
\begin{equation}
    Pr(|\phi_i-\hat\phi_i|\geq\epsilon)\leq\exp(-(2q_j-q_j^2)T h(\frac{\epsilon}{(2q_j-q_j^2)\delta_i}))
    \end{equation}
where
\begin{equation}
    h(x)=(1+x)\ln(1+x)-x.
\end{equation}
\begin{rem}
    Note that the error analysis has been calculated for fixed client $i$ in this paper. In terms of the vector $\Phi=(\phi_1,\phi_2,\cdots,\phi_n)$ in different $l^p$ Banach spaces where $1\leq p\leq\infty$, a general method is given to bound the estimation error. We have
    \begin{equation}
    \begin{aligned}
    Pr(||\Phi-\hat\Phi||_p\geq\epsilon)&=&&1-Pr(\underset{i=1}{\stackrel{n}{\sum}}|\phi_i-\hat{\phi}_i|^p\leq\epsilon^p)\\
    &\leq&&1-Pr(\underset{i=1}{\stackrel{n}{\bigcap}}|\phi_i-\hat{\phi}_i|\leq\frac{\epsilon}{\sqrt[p]{n}})\\
    &=&&Pr(\underset{i=1}{\stackrel{n}{\bigcup}}|\phi_i-\hat{\phi}_i|\geq\frac{\epsilon}{\sqrt[p]{n}})\\
    &\leq&&\underset{i=1}{\stackrel{n}{\sum}}Pr(|\phi_i-\hat{\phi}_i|\geq\frac{\epsilon}{\sqrt[p]{n}})
    \end{aligned}
    \end{equation}
    if $1\leq p<\infty$ and
    \begin{equation}
    \begin{aligned}
    Pr(||\phi-\hat\phi||_p\geq\epsilon)
    &=&&Pr(\underset{i=1}{\stackrel{n}{\bigcup}}|\phi_i-\hat{\phi}_i|\geq\epsilon)\\
    &\leq&&\underset{i=1}{\stackrel{n}{\sum}}Pr(|\phi_i-\hat{\phi}_i|\geq\epsilon)
    \end{aligned}
    \end{equation}
    if $p=\infty$.
\end{rem}
\subsection{Classification Truncated Monte Carlo}
Classification Truncated Monte Carlo sampling(CTMC) is a combination method of TMC and CMC. We choose clients from each class with the same proportion, then we truncated the clients if they bring less marginal contributions than the truncated factor. The pseudocode is shown in Algorithm \ref{ctmcalgo}.
\begin{algorithm}[htb]
  \caption{ the CTMC algorithm.}\label{ctmcalgo}
  \begin{algorithmic}[1]
    \Require
      Train data $N=\{1,2,\ldots,n\}$, utility function $U$
    \Ensure
      Shapley value $\phi_i$($i=1,2,\ldots,n$)
    \While {Convergence criteria not met} \do;
    \State Select clients from each class with the same proportion $q$. The selected set is $P$ and the permutation group is $\Gamma$;
    \State $t:=t+1$;
    \State $\gamma^t$: Random permutation in $\Gamma$;
    \State $v_0^t=0$;
    \For { $j\in\{1,2,\ldots,\lfloor qn\rfloor\}$ } \do;
    \If {$|U(\gamma^t(P))-v_{j-1}^t|<\mbox{Truncated factor}$}
    \State{$v_j^t=v_{j-1}^t$}
    \Else \State $v_j^t=U(\gamma^t[1],\gamma^t[2],\cdots,\gamma^t[j])$;
    \EndIf
    \EndFor
    \State $\phi_{\gamma^t[j]}=\frac{t-1}{t}\phi_{\gamma^{t-1}[j]}+\frac{1}{t}(v_j^t-v_{j-1}^t)$;
    \EndWhile
    \State \Return $\phi_1,\phi_2,\cdots,\phi_n$;
  \end{algorithmic}
\end{algorithm}
\section{Experiment}

In this paper, we utilize the Wine, Cancer, and Adult datasets to evaluate the performance of our proposed method. These datasets are commonly used benchmarks for classification tasks in machine learning and have been extensively studied in previous research.

In order to evaluate the effectiveness of our method, we set up two experimental settings in this section. The first experiment involves evaluating each dataset using the raw data without any modifications or additions. The second experiment involves introducing partial label noise to the datasets and observing the detection of these noise samples. Partial label noise occurs when some of the labels in the dataset are incorrect, which can occur due to human error or faulty measurement equipment. This is a common problem in real-world datasets, and it is important to evaluate the robustness of our model to these types of errors. The CPU we use is the Intel(R)Xeon(R)Silver4108 CPU.

{\bf Datasets:} We consider {\bf Wine}, {\bf Cancer}, and {\bf Adult} datasets. The {\bf Wine} dataset contains 178 samples and 3 classifications, each with a capacity of 59, 71, and 48. Each sample has 13 features. We evaluate half of the data points and use 49 data points to assess the performance of classifiers. The last 40 held-out data points are used to evaluate the final data valuation performance, where we plot performance as a function of the amount of data deleted. 

The {\bf Cancer} dataset contains 286 samples, of which 201 are of one class and 85 are of another. Each sample has nine features. We evaluate half of the data points and use 43 data points to assess the performance of classifiers. The last 100 held-out data points are used to evaluate the final data valuation performance. 

The {\bf Adult} dataset contains 48,842 samples, each with 15 features. We followed \cite{yoon2020data}, evaluate 200 data points, and use 200 data points to assess the value of the data. The remaining data points are held-out data points, where they are used to evaluate the final data valuation performance.

{\bf Classifier:} We use the logistic regression algorithm with liblinear for quantifying the value of the subset. For simulating the ordinal scenarios, we give an extra sample weight in position $i$ as shown in Figure \ref{weight}, i.e.
\begin{equation}
    W_i=\frac{n}{\sqrt{2\pi}\sigma}\exp(-\frac{(i-\mu)^2}{2\sigma^2})
\end{equation}
where $\mu=\frac{n-1}{2}$ and $\sigma=\frac{n-1}{6}$.
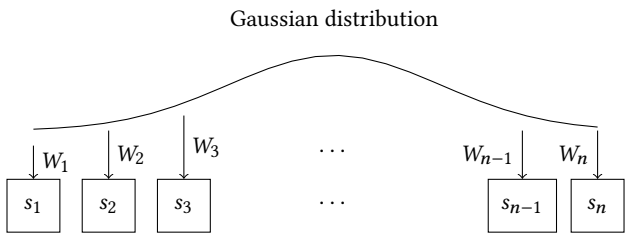
\begin{figure}
\begin{tikzpicture}
\tikzset{
box/.style ={
rectangle, 
minimum width =20pt,
minimum height =20pt, 
inner sep=5pt, 
draw=black 
}}
\tikzset{
boxx/.style ={
rectangle, 
minimum width =50pt,
minimum height =20pt, 
inner sep=5pt, 
draw=white 
}}
\draw[domain =-4:3.5] plot (\x ,{ exp(-(0.5*\x)^2)});
\node[boxx](8) at(0,1.5){Gaussian distribution};
\node[box] (1) at(-4,-1){$s_1$};
\node[box] (2) at(-3,-1){$s_2$};
\node[box] (3) at(-2,-1){$s_3$};
\draw[->] (-4,-0.2) --node[right]{$W_1$}(1);
\node[boxx] (4) at(0,-1){$\cdots$};
\node[boxx] (5) at(0,-0.3){$\cdots$};
\draw[->] (-3,0) --node[right]{$W_2$}(2);
\draw[->] (-2,0.2) --node[right]{$W_3$}(3);
\node[box] (6) at(2.5,-1){$s_{n-1}$};
\node[box] (7) at(3.5,-1){$s_n$};
\draw[->] (2.5,-0) --node[left]{$W_{n-1}$}(6);
\draw[->] (3.5,-0) --node[left]{$W_n$}(7);
\end{tikzpicture}
\caption{Weight in ordinal scenarios}\label{weight}
\end{figure}

\begin{table*}[t]
  \caption{Average AUC(Up to $50\%$ data removed). Within a consistent experimental setting, a smaller area signifies better performance. On a fixed dataset, the area corresponding to noisy data is smaller than that of the raw data. This phenomenon can be attributed to the inferior performance of classifiers trained on noisy datasets compared to those trained on raw datasets.}
  \label{tab:area}
  \begin{tabular}{cccl}
    \toprule
    Dataset & TMC(Raw/Noisy) & CMC(Raw/Noisy) & CTMC(Raw/Noisy) \\
    \midrule
    Wine &8.51/6.42 & 8.29/6.38 & 8.41/6.32 \\
    Cancer  & 9.37/6.60 & 9.37/6.62 & 9.35/6.62 \\
    Adult & 6.64/4.81 & 6.60/4.87 & 6.66/4.83 \\
    \bottomrule
  \end{tabular}
\end{table*}

{\bf Evaluating methods}
We use two benchmarks, evaluating by removing high-value data points and noisy-label detection. We repeat the experiment five times and report the average values.

{\bf Experiment details} During each evaluation of the subset performance, the classifier is reinitialized and trained until convergence, with the accuracy of the resulting classifier on held-out data serving as the value of the subset. To trade off algorithm performance and computational cost, a ratio of 0.8 is chosen for both the CMC and CTMC algorithms. Additionally, a Truncated factor of 0.05 is consistently used for all settings.

\subsection{Removing high-value data points}\label{subsec:removinghighvalue}
The removal of high-value samples can significantly impact the performance of classifiers, thus high-value sample removal experiments are often used as an evaluation method for data pricing. In our study, we conducted high-value sample removal experiments on the original dataset.

We remove data points from high-value to low-value. After removing data points we retrain the classifier with the remaining data points and evaluate the classifier's performance on the held-out data. Results of removing high-value data points experiment on three datasets are shown in Figure \ref{fig:removing}, and the area under the curve(AUC) labeled with Raw are shown in Table \ref{tab:area}. As the sample size decreased, each algorithm shows a significant performance degradation. The performance among the three algorithms is similar. We also show the training time of three algorithms on three datasets in Table \ref{tab:time}(labeled with Raw). It is worth noting that, the CTMC algorithm is faster than the TMC and the CMC algorithms.  The comparison of the improvement rate of the CTMC algorithm over the TMC algorithm on the original dataset is displayed in Table \ref{tab:time}. We refer to the specific details of the evaluation method presented in \cite{Zou}.

\begin{figure*}
  \centering
  
  \subfigure[wine]{
		\includegraphics[width=0.32\linewidth]{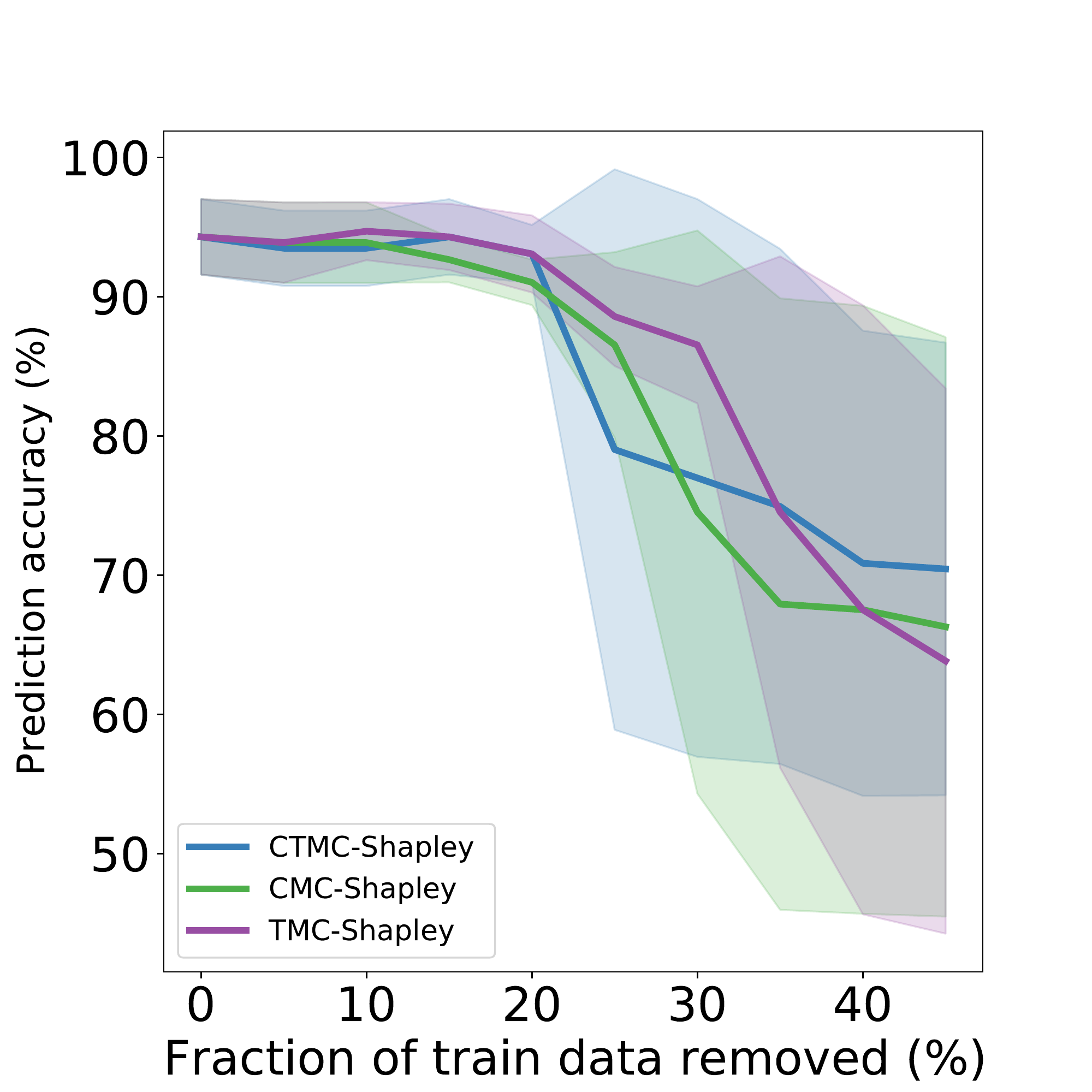}}
  \subfigure[cancer]{
		\includegraphics[width=0.32\linewidth]{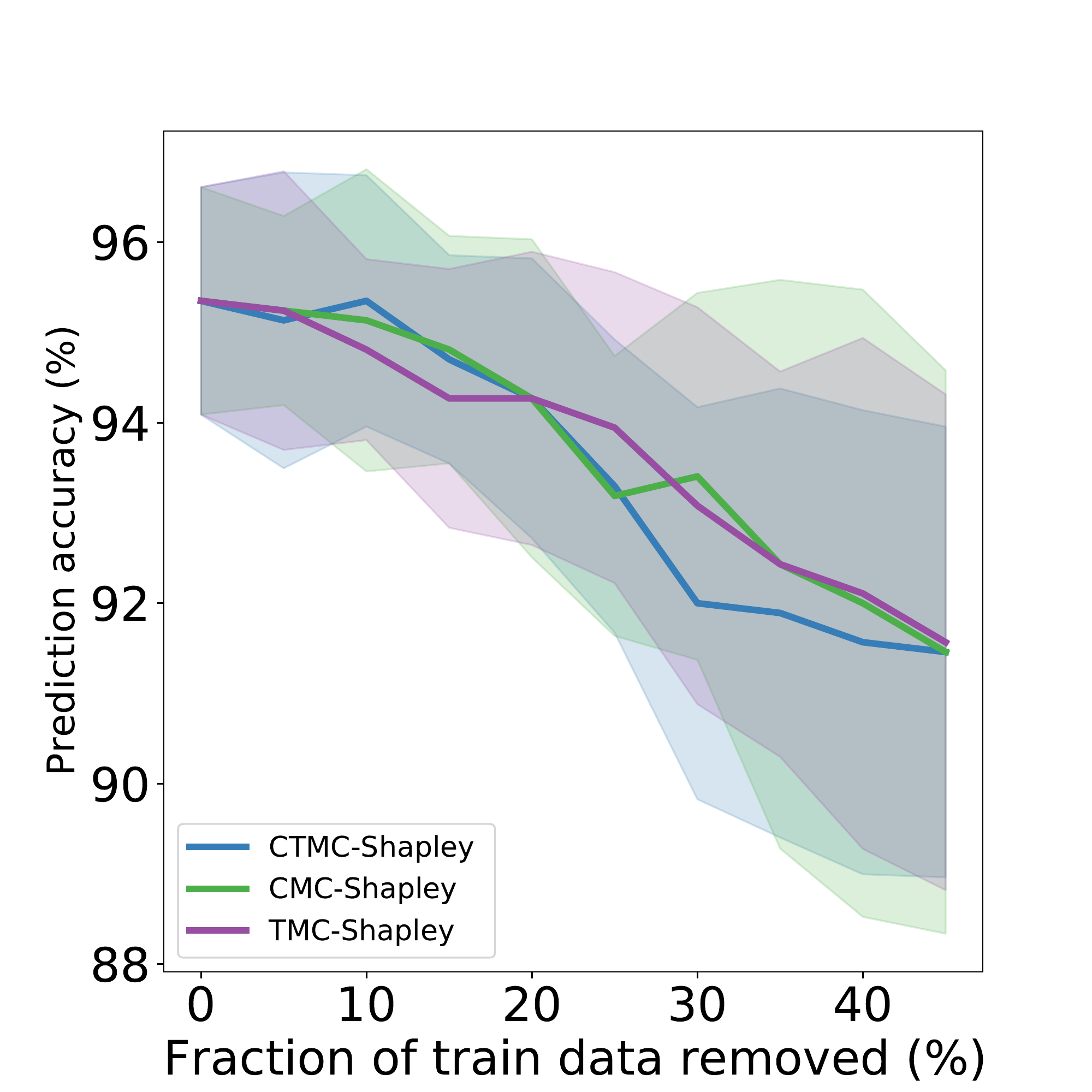}}
  \subfigure[adult]{
		\includegraphics[width=0.32\linewidth]{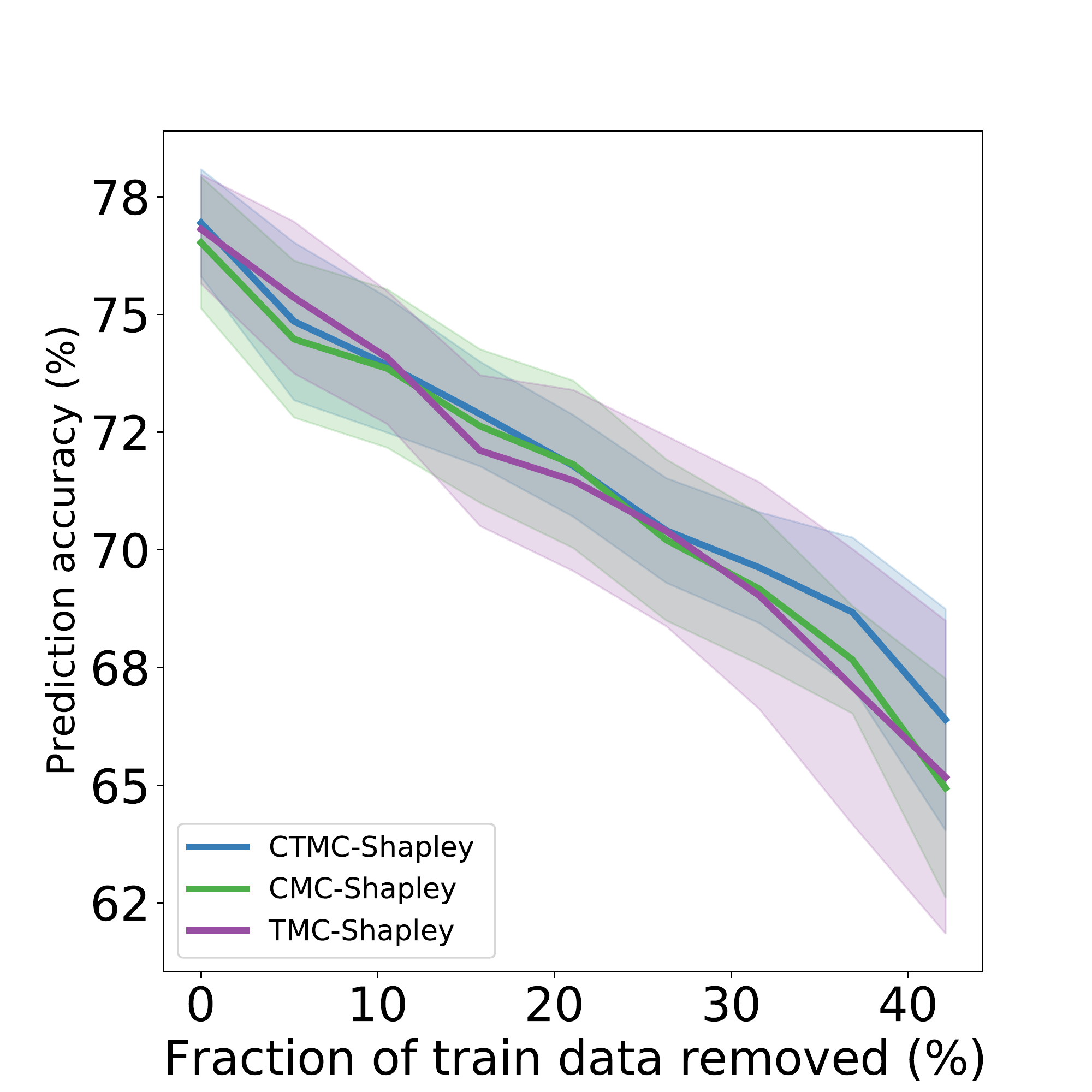}}

  \caption{Results for removing high-value data points experiment on three datasets. In this figure, the lines represent the mean values and the shadows correspond to the standard deviation.}
  \label{fig:removing}
\end{figure*}

\subsection{Noisy-label detection}
In some scenarios, there might be some mislabeled data points in the dataset, so a method that can distinguish between clean and mislabeled data is necessary. Data valuation is a canonical method for recognizing mislabeled data points since mislabeled data points lead to a lower data value. Compare to clean data, mislabeled data has a negative impact on model performance. Then we can obtain the model with better performance by deleting the mislabeled data. 

We synthesize the noisy dataset by mislabeling $20\%$ of the data. Three algorithms show similar performance and deleting data from low-value samples causes the performance of the model to improve and then degrade. We still show in Figure \ref{fig:detection} how the performance of the classifier changes with the proportion of deleted samples. The least represents deletion from low-value samples to high-value samples, and the most represents deletion from high-value samples to low-value samples. Moreover, the AUC(labeled with Noisy) are shown in Table \ref{tab:area}. Similar to Subsection \ref{subsec:removinghighvalue}, the three algorithms show consistent performance. We also show the training time of three algorithms on three datasets in Table \ref{tab:time}(labeled with Noisy). Except that the CMC and the CTMC algorithms have similar training time in Wine dataset, the CTMC algorithm is faster than the TMC and the CMC algorithms in all of the other situations. The improvement rate of the CTMC algorithm over the TMC algorithm on the noisy dataset is presented in Table \ref{tab:time}. On noisy datasets, the CTMC algorithm demonstrates a higher reduction in computational costs than the TMC algorithm as compared to original datasets. Notably, it achieves a speedup of $17.28\%$ and $18.79\%$ on Wine and Adult datasets, respectively.

\begin{figure*}
  \centering
    
	\subfigure[wine]{
		\includegraphics[width=0.32\linewidth]{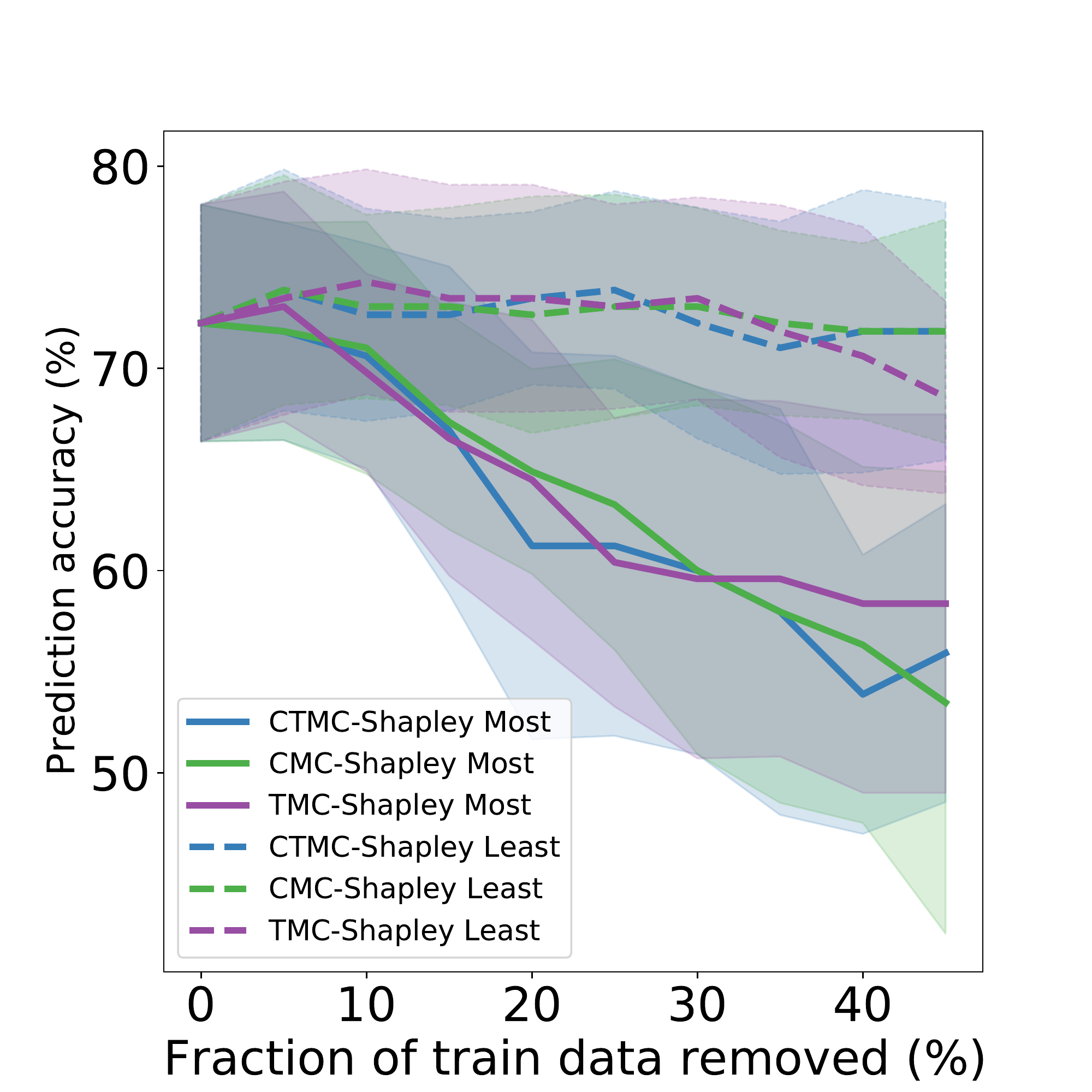}}
    \subfigure[cancer]{
		\includegraphics[width=0.32\linewidth]{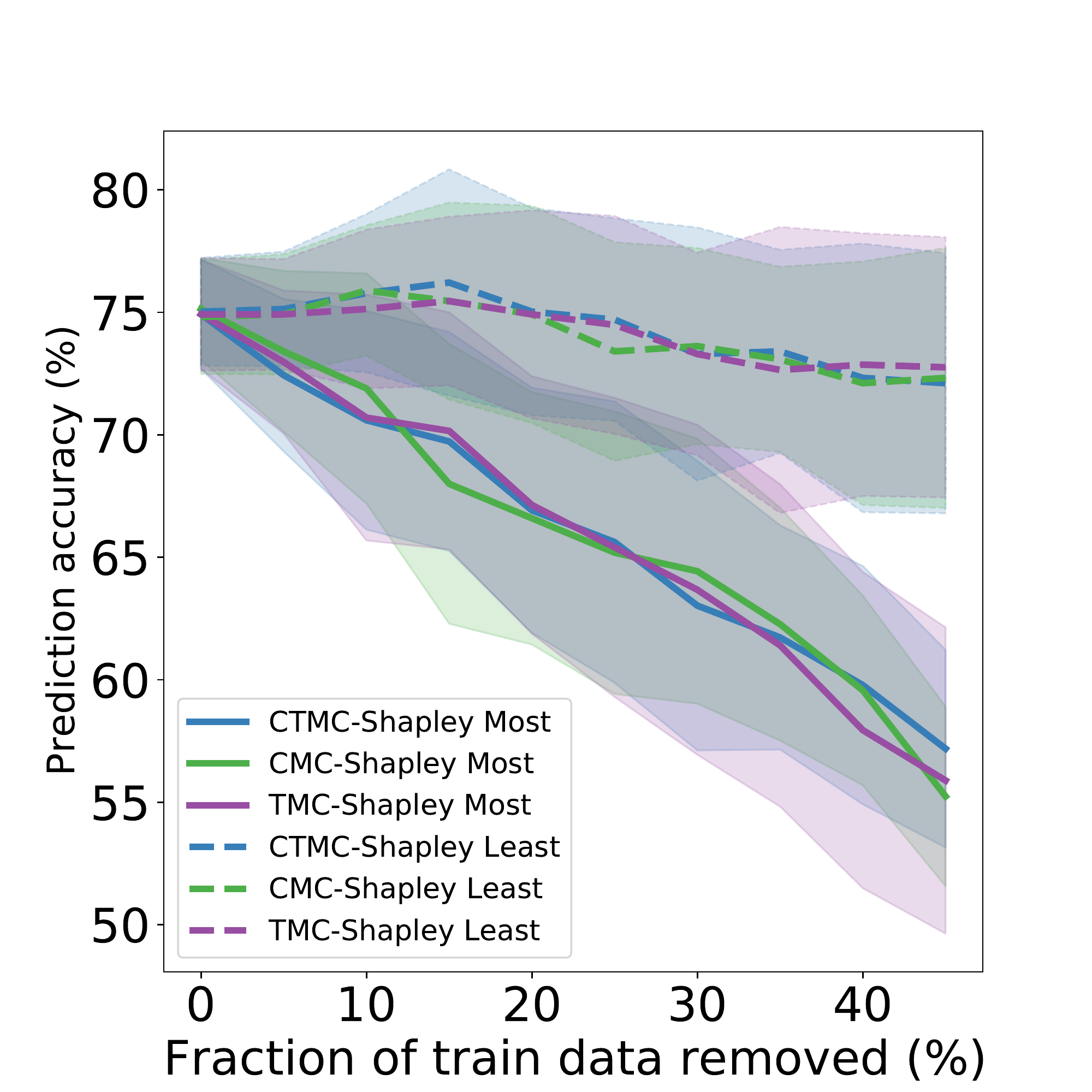}}
    \subfigure[adult]{
		\includegraphics[width=0.32\linewidth]{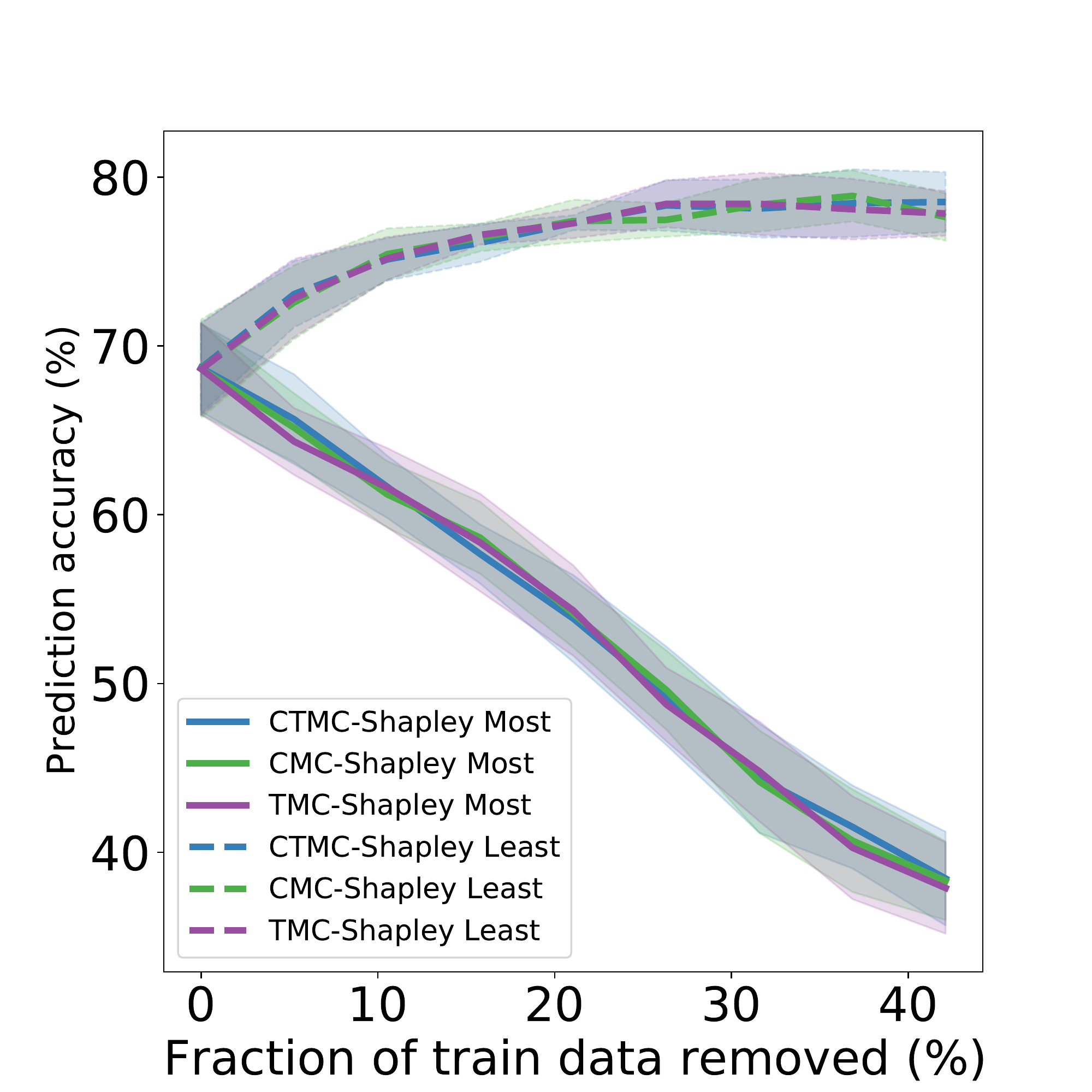}}

  \caption{Results for noisy-label detection experiment on three datasets. In this figure, the lines represent the mean values and the shadows correspond to the standard deviation.}
  \label{fig:detection}
\end{figure*}

\begin{table*}[t]
  \caption{Average training time(Seconds). On a fixed dataset, the training time for noisy data is slower than that for the raw data. The reason is that noisy data leads to the inferior performance of classifiers.}
  \label{tab:time}
  \begin{tabular}{ccccl}
    \toprule
    Dataset & TMC(Raw/Noisy) & CMC(Raw/Noisy) & CTMC(Raw/Noisy) & Time reduce($\%$)\\
    \midrule
    Wine &167.97/376.62 & 203.17/310.41 & 148.05/311.54 & 11.86/17.28 \\
    Cancer  & 1214.88/3643.38 & 2560.34/4750.74 & 1142.45/3250.64 & 5.96/10.78\\
    Adult & 7484.78/13199.39 & 8051.50/12500.36 & 6861.82/10719.87 & 9.32/18.79\\
    \bottomrule
  \end{tabular}
\end{table*}

\section{Conclusion}
This work develops the classic Shapley value to the partial ordinal Shapley value and extends the TMC algorithm to the ordinal setting. Besides, the CMC and the CTMC algorithms are provided based on the assumption that data points in the same class bring similar information. In Wine noisy dataset, both the CMC and the CTMC algorithms are faster than the TMC algorithm where the CMC and the CTMC algorithms have similar training time. In other situations, the CTMC algorithm is faster than the TMC and the CMC algorithms.


\begin{appendices}
   
\section{Appendix: Calculation of error analysis}\label{appendix}
The goal of this appendix is to provide the error estimations of the TMC and the CMC algorithms. We estimate the error for fixed client $i$. If it does not create ambiguities, we will allow the abuse of notation of denoting $\phi_i$ as $\phi$. 

{\bf TMC algorithm}

The method we use in this section is similar to the method in \cite{Maleki}. Let $m_j$ be the number of samples where client $i$ is in $j$-th place.
By Hoeffding's inequalities, we have
\begin{equation}
    Pr(|\phi^{j}-\hat{\phi}^{j}|\geq\frac{\epsilon}{k})\leq2\exp(-\frac{2m_j\epsilon^2}{k^2 r_j^2}).
    \end{equation}
    Sum over all the case of $j\leq k$, we have
    \begin{equation}
    \begin{aligned}
    Pr(|\phi^{(k)}-\hat{\phi}^{(k)}|\geq\epsilon)&\leq&& \underset{j=1}{\stackrel{k}{\sum}}Pr(|\phi^j-\hat{\phi}^j|\geq\frac{\epsilon}{k})\\
    &\leq&& \underset{j=1}{\stackrel{k}{\sum}}2\exp(-\frac{2m_j\epsilon^2}{k^2 r_j^2}).
    \end{aligned}
    \end{equation}
     Since
    \begin{equation}
        |\phi-\hat\phi^{(k)}|\leq|\phi^{(k)}-\hat\phi^{(k)}|+|\phi-\phi^{(k)}|\leq|\phi^{(k)}-\hat\phi^{(k)}|+\epsilon_k
    \end{equation}
    where we assume $\epsilon\geq\epsilon_k$, 
    we have
    \begin{equation}\label{tmcerror}
    \begin{aligned}
        Pr(|\phi-\hat{\phi}^{(k)}|\geq\epsilon)&\leq&& Pr(|\phi^{(k)}-\hat{\phi}^{(k)}|\geq\epsilon-\epsilon_k)\\
        &\leq&&\underset{j=1}{\stackrel{k}{\sum}}2\exp(-\frac{2m_j(\epsilon-\epsilon_k)^2}{k^2 r_j^2})\\
        &\leq&&\underset{j=1}{\stackrel{k}{\sum}}2\exp(-\frac{2m_j(\epsilon-\epsilon_k)^2}{k^2 r_{max,k}^2}).
    \end{aligned}
    \end{equation}
    \begin{rem}
    We consider estimating errors with known boundings of utility functions of the same size. The classic Shapley value situation is calculated in \cite{Maleki} where the boundings are independent of the length of $S$. In this paper, we consider $a_k,b_k\in\mathbb{R}^+$ with $k=1,2,\ldots,n$ such that
    \begin{equation}
    a_k|S|\leq U(\pi|_S(\mathfrak{S}))\leq b_k|S|,|S|=k,\pi|_S\in\Pi|_S.
    \end{equation}
    The range of marginal contribution in position $k+1$ is
    \begin{equation}\label{tmcerrorstrata}
    r_{k+1}
    \leq(b_{k+1}-a_k)k+b_{k+1}-(a_{k+1}-b_k)k-a_{k+1}.
    \end{equation}
    Then the boundings can be calculated using the same method as the former situation, i.e. the range of the position $j$ in \eqref{tmcerror} can be replaced by \eqref{tmcerrorstrata} after renaming.
    
    Moreover, if we add another assumption that the marginal contribution is non-negative, the range of the marginal contribution in position $k+1$ is
    \begin{equation}
    r_{k+1}\leq(b_{k+1}-a_k)k+b_{k+1}-c_{k+1}
    \end{equation}
    where
    \begin{equation}
    c_{k+1}=\max\{(a_{k+1}-b_k)k+a_{k+1},0\}.
    \end{equation}
    
    \end{rem}

{\bf CMC algorithm} 

    The method we use is similar to \cite{jia2019towards} where we use Bennett inequalities\cite{Bennett} for error estimation. The principle underlying the error analysis of the CMC algorithm is that, since only a subset of data is selected from each class at each round, each data point within the dataset has only a certain probability of being chosen. By utilizing the law of total variance, we can narrow the range of variance and thereby achieve the goal of reducing the scope of error estimation.

    We denote $\Delta$ be the indicator of whether the client $i$ has been chosen or not, i.e.
    \begin{equation}
        Pr(\Delta=1)=q_j,\quad Pr(\Delta=0)=1-q_j.
    \end{equation}
    We have
    \begin{equation}
    Var(\delta_i)=E[Var(\delta_i)]+Var(E[\delta_i]),
    \end{equation}
    and
    \begin{multline}
    E[Var(\delta_i)]=(1-q_j)E[Var(\delta_i)|\delta_i=0]+q_jE[Var(\delta_i)|\delta_i\neq0]\\
    \leq q_j\delta_i^2.
    \end{multline}
    Since
    \begin{multline}
    Var(E[\delta_i])=(1-q_j)(E[\delta_i|\delta_i=0]-E[\delta_i])^2\\
    +q_j(E[\delta_i|\delta_i\neq0]-E[\delta_i])^2
    \end{multline}
    where
    \begin{equation}
    E[\delta_i]=(1-q_j)E[\delta_i=0]+q_jE[\delta_i\neq0]
    =q_jE[\delta_i\neq0],
    \end{equation}
    we have
    \begin{equation}
        Var(E[\delta_i])=q_j(1-q_j)E[\delta_i\neq0]\leq q_j(1-q_j)\delta_i^2.
    \end{equation}
    It leads to
    \begin{equation}
    Var[\delta_i]\leq(2q_j-q_j^2)\delta_i^2.
    \end{equation}
    Using Bennett inequalities, we have
    \begin{equation}
    Pr(|\phi-\hat\phi|\geq\epsilon)
    \leq\exp(-(2q_j-q_j^2)T h(\frac{\epsilon}{(2q_j-q_j^2)\delta_i}))
    \end{equation}
where
\begin{equation}
    h(x)=(1+x)\ln(1+x)-x.
\end{equation}

\end{appendices}

\balance

\clearpage
\bibliographystyle{ACM-Reference-Format}
\bibliography{sample}
\end{document}